\documentclass[runningheads]{llncs}
\usepackage{graphicx}
\usepackage{amsmath}
\usepackage{amssymb}
\usepackage{booktabs}
\usepackage{tcolorbox}
\usepackage{tikz}
\usetikzlibrary{arrows}
\usetikzlibrary{positioning}
\usetikzlibrary{external}
\usepackage{pgfplots}
\pgfplotsset{compat=1.16}
\usepackage{easy-todo}
\usepackage{hyperref}
\hypersetup{
    colorlinks=true,
    linkcolor=blue,
    filecolor=magenta,      
    urlcolor=cyan,
    pdftitle={Overleaf Example},
    pdfpagemode=FullScreen,
    }
\usepackage[nameinlink]{cleveref}

\usepackage{caption} 
\begin{document}
\title{Exact Non-Oblivious Performance of Rademacher Random Embeddings}
%
%
\author{Maciej Skorski\orcidID{0000-0003-2997-7539} \and
Alessandro Temperoni\orcidID{0000-0003-0272-6596}}

\authorrunning{M. Skorski and A. Temperoni}
%
\institute{University of Luxembourg, 4365 Esch-sur-Alzette, Luxembourg  \email{@uni.lu}}
%
\maketitle              
\begin{abstract}
This paper revisits the performance of 
Rademacher random projections, establishing novel statistical guarantees that are numerically sharp and non-oblivious with respect to the input data.

More specifically, the central result is the Schur-concavity property of Rademacher random projections with respect to the inputs. This offers a novel geometric perspective on the performance of random projections, while improving quantitatively on bounds from previous works. As a corollary of this broader result, we obtained the improved performance on data which is sparse or is distributed with small spread. This non-oblivious analysis is a novelty compared to  techniques from previous work, and bridges the frequently observed gap between theory and practise.

The main result uses an algebraic framework for proving 
Schur-concavity properties, which is a contribution of independent interest and an elegant alternative to derivative-based criteria. 

\end{abstract}

\section{Introduction}

\subsection{Background and Related Work}

The seminal result of Johnson and Lindenstrauss~\cite{johnson1984extensions} quantifies the amazing performance of \emph{random linear maps as embeddings}: they map input data into a much smaller dimension (compression) while nearly maintaining the scalar products (almost isometrically). The \emph{strong quantitative guarantees}, a distinctive feature of this technique, enabled its application to a range of problems and areas, including nearest-neighbour search \cite{kleinberg1997two}, clustering \cite{dasgupta1999learning,cohen2015dimensionality,makarychev2022performance}, outlier detection \cite{aouf2012approximate}, ensemble learning \cite{cannings2017random}, adversarial machine learning~\cite{vinh2016training}, feature hashing in machine learning \cite{jagadeesan2019understanding}, numerical linear algebra \cite{sarlos2006improved,cohen2015dimensionality,clarkson2017low}, convex optimization \cite{zhang2013recovering}, differential privacy \cite{blocki2012johnson}, neuroscience \cite{ganguli2012compressed} and numerous others; for a comprehensive literature review we refer the reader to the recent survey~\cite{freksen2021introduction}.

The focus of this work is on numerical guarantees for the \emph{almost distance-preserving} property, which is formally stated as
\begin{align}\label{eq:near_isometry_informatl}
    \| \Phi x \|^2 \approx \|x\|^2\quad \textrm{with high probability},
\end{align}
where an appropriately sampled matrix $\Phi \in \mathbb{R}^{m\times n}$ represents the projection of an $n$-dimensional vector $x$ to $m$ dimensions ($m\ll n$ is desired), and the relative approximation error in \eqref{eq:near_isometry_informatl} is referred to as \emph{distortion}.

The long-line of research~\cite{frankl1988johnson,indyk1998approximate,achlioptas2003database,ailon2006approximate,matouvsek2008variants,dasgupta2010sparse,kane2014sparser,jagadeesan2019understanding,skorski2021johnson,skorski2022johnson} has incrementally established various guarantees for \eqref{eq:near_isometry_informatl}, in the form of \emph{distortion-confidence trade-offs}: while a small distortion ensures that the analytical task can be performed with a similar effect over the embedded data, high confidence guarantees with what probability it will happen.
Yet the quantitative analysis of the property \eqref{eq:near_isometry_informatl} has remained a difficult challenge, resulting in complex proofs simplified many times~\cite{cohen2018simple,jagadeesan2019simple}, crude statistical bounds (for example, sparse variants have an exponential gap with respect to the sharp no-go results~\cite{burr2018optimal}), and a lack of finite-dimensional insights (bounds are input-oblivious which widens the gap between theory predictions and empirical performance~\cite{venkatasubramanian2011johnson,fedoruk2018dimensionality}).

This work addresses the aforementioned gap by revisiting the most promising construction of \emph{Rademacher random projections}, which uses the following matrix
\begin{align}\label{eq:rademacher_matrix}
\Phi_{k,j} = \frac{1}{\sqrt{m}}r_{k,i},\quad r_{k,i}\sim^{IID} \{-1,+1\}.
\end{align}
More specifically, this paper solves the following problem:
\begin{center}
\begin{tcolorbox}[arc=3mm, colback=white]
Give a precise, non-asymptotic, non-oblivious analysis of random projections \eqref{eq:rademacher_matrix}.
\end{tcolorbox}
\end{center}
As recently~\cite{burr2018optimal} showed, Rademacher random projections are asymptotically \emph{dimension-optimal with exact constant}; this result improves upon a previous suboptimal bound of Kane and Nelson~\cite{kane2014sparser}. The statistical performance of Rademacher projections is superior to the sparse ones, as demonstrated empirically in \ref{fig:distortionVSembedsparsity}. Furthermore, the theoretical bounds for Rademacher random projections are much better than those available for sparse analogues~\cite{cohen2018simple}. The best, prior to this paper, analysis of \eqref{eq:rademacher_matrix} is given by Achlioptas in \cite{achlioptas2003database}. It is worth noting that Rademacher projections are also superior to their Gaussian counterparts; indeed, we know that they are dominated by the gaussian-based projections~\cite{achlioptas2003database}. The relation of statistical performance and input structure has not been understood in-depth yet; as for conceptually similar research, we note that recent results show that for sparse data one can improve the sparsity of random projections, gaining in computing time ~\cite{jagadeesan2019understanding,skorski2022johnson}.

\begin{figure}
\centering
\resizebox{.6\textwidth}{!}{
\begin{tikzpicture}
\begin{axis}[
    axis x line*=bottom,
    axis y line*=left,
    xtick={-1,0,1},
    ymode=log,
    xlabel={threshold $\epsilon$},
    ylabel={$\mathbf{P}\{distortion>\epsilon\}$},
    legend style={
        at={(0.98,0.98)},
        anchor=north east,
        mark=none,
    }
]
\addplot table [x=x, y=0.1, col sep=comma, mark=none] {ecdfs.csv};
\addplot table [x=x, y=0.25, col sep=comma, mark=none] {ecdfs.csv};
\addplot table [x=x, y=1.0, col sep=comma, mark=none] {ecdfs.csv};
\legend{$p=0.10$, $p=0.25$, $p=1.00$};
\end{axis}
\end{tikzpicture}}
\caption{The distortion tail (empirical CCDF) w.r.t. the embedding density parameter
$p$ which equals the fraction of non-zero elements in the matrix. 
Dense Rademacher projections ($p=1$) are numerically superior to their sparse counterparts ($p<1$). The comparison was done on a toy dataset.
}
\label{fig:distortionVSembedsparsity}
\end{figure}
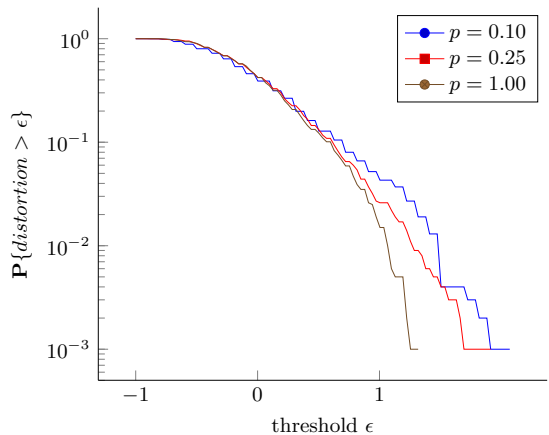

\subsection{Our Contribution}

Our study of the stochastic behavior of \eqref{eq:near_isometry_informatl} offers the following novel contributions:
\begin{description}
\item[(a)] \textbf{non-oblivious insights}, by quantifying the \emph{dependency on input spreadness}. Loosely speaking, we prove that more spread-out input data lead to heavier tails in the distributed distortion. In particular\footnote{We can think of sparse input as being the extreme opposite of the "well-spread" property.}, we obtain the \emph{improved performance on sparse input data}.
\item[(b)] \textbf{Schur-concavity framework}, used to provide the missing geometric intuitions for the performance of random projections. 
\item[(c)] \textbf{numerically optimal bounds}, which precisely capture the extreme behavior and cannot be improved by any constant. 
\item[(d)] \textbf{benchmarking} against previous bounds, both theoretically and empirically. 
By means of distortion, we measured how high-dimensional vectors are projected with different density to prove the strength of our bounds. 
\end{description}

\subsection{Organization}
The remainder of the paper is organized as follows:
\Cref{sec:prelim} introduces notation and technical notions used throughout the paper; then \Cref{sec:contrib} discusses and benchmarks novel results of this paper, and \Cref{sec:conclusion} concludes the work.
Parts of the theoretical analysis appear in the Appendix \ref{sec:proofs}. 

\begin{figure}
\centering
\resizebox{.6\textwidth}{!}{
\begin{tikzpicture}
\begin{axis}[
    axis x line*=bottom,
    axis y line*=left,
    xtick={0.2,0.4,0.6,0.8,1},
    ymode=log,
    xlabel={embedding density $p$},
    ylabel={quadratic chaos moment},
    legend style={
        at={(0.02,0.98)},
        anchor=north west
    },
]
\addplot table [x=p_value, y=0.2, col sep=comma] {quadratic_chaos.csv};
\addplot table [x=p_value, y=0.4, col sep=comma] {quadratic_chaos.csv};
\addplot table [x=p_value, y=0.6, col sep=comma] {quadratic_chaos.csv};
\addplot table [x=p_value, y=0.8, col sep=comma] {quadratic_chaos.csv};
\addplot table [x=p_value, y=1, col sep=comma] {quadratic_chaos.csv};
\legend{$sparsity=0$,$sparsity=0.2$,$sparsity=0.3$,$sparsity=0.4$,$sparsity=0.5$};
\end{axis}
\end{tikzpicture}}
\caption{Data-aware insights. We measured the theoretical moments against \emph{input data sparsity}. The plot shows that increased sparsity in the input data leads to smaller values of the moments. This result becomes more evident when related to the embeddings density parameter $p$. }
\end{figure}
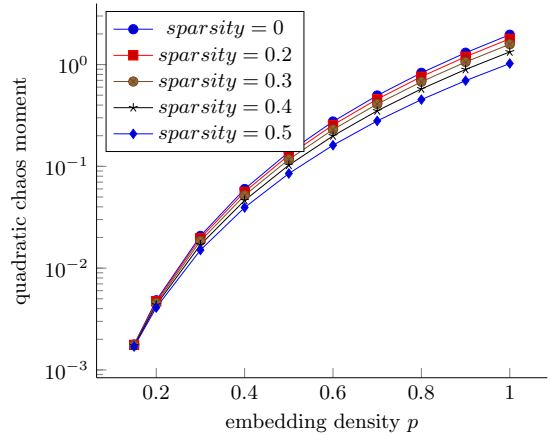


\section{Preliminaries}\label{sec:prelim}

We start by recalling some basic concepts from probability theory, algebra, combinatorics and optimization. 

Throughout the paper we work basic probability distributions: normal, binomial, and Rademacher. Given a vector $x$ we denote by $\|x\|$ its euclidean norm and by $\|x\|_0$ the number of its non-zero components; we also say that $x$ is $\ell$-sparse with $\ell=\|x\|_0$, for example $x=\left(\frac{1}{4},\frac{1}{4},\frac{1}{4},0\right)$ is $3$-sparse.

\emph{Partitions} represent a positive integer $q$ as a sum of positive integers. A non-decreasing and non-negative sequence $\lambda$ is a partition of $q$, denoted as $\lambda \vdash q$, when $\sum_i \lambda_i = q$; in the \emph{frequency notation} distinct parts are assigned frequencies, so that $\lambda = 1^{f_1}\ldots q^{f_q}$ where $\sum_i i f_i = q$. For example, we have $\lambda = 1^2 2  = (2,1,1)\vdash 4$. 

\emph{Monomial symmetric functions} for a given partition $\lambda =(\lambda_1,\ldots,\lambda_k)\vdash q$ are defined as the sum of all distinct monomials with exponent $\lambda$, that is
$\mathbf{m}_{\lambda}(x) = \sum_{i_1,\ldots,i_k}x_{i_1}^{\lambda_1}\cdots x_{i_k}^{\lambda_k}$
where $i_1,\ldots,i_k$ and monomials
$x_{i_1}^{\lambda_1}\cdots x_{i_k}^{\lambda_k}$ are distinct. For example, we have $\mathbf{m}_{(2,1,1,1)}(x_1,x_2,x_3,x_4) = x_{1}^{2} x_{2} x_{3} x_{4} + x_{1} x_{2}^{2} x_{3} x_{4} + x_{1} x_{2} x_{3}^{2} x_{4} + x_{1} x_{2} x_{3} x_{4}^{2}$. 

\emph{Elementary symmetric functions} are defined as $\mathbf{e}_k(x) = \sum_{i_1<\ldots<i_k}x_{i_1}x_{i_2}\cdots x_{i_k}$. Both monomial and elementary polynomials are a base of symmetric polynomials. For more details, we refer to the monographs \cite{alexanderson2020}.

The \emph{majorization order} on $n$-dimensional vectors is defined as follows:
we say that $x$ is dominated by $y$, denoted by $x\prec y$, 
if for their non-decreasing rearrangements $(x_i^\downarrow)$
and $(y_i^\downarrow)$ we have the inequality
$
 \sum_{i=1}^k x_i^\downarrow \geqslant \sum_{i=1}^k y_i^\downarrow$ for $k=1,\dots,n$ with equality when $k=n$; equivalently, $x$ can be produced from $y$ by a sequence of \emph{Robin-Hood operations}
which replace $x_i>x_j$ by $x_i\gets x_{i}-\epsilon,x_j \gets x_j+\epsilon$ for $\epsilon \in \left(0,\frac{x_i-x_j}{2}\right)$. 
 Intuitively, $x$ being dominated by $y$ means that $x$ is more spread-out/dispersed than $y$. For example, we have $\left(\frac{1}{4},\frac{1}{4},\frac{1}{4},\frac{1}{4}\right)\prec \left(\frac{1}{3},\frac{1}{3},0,\frac{1}{3}\right)$ (it takes 3 Robin-Hood transfers). 
 
 The \emph{Schur-convexity} of a function $f:\mathbb{R}^n\to \mathbb{R}$ is the following property: $x\prec y$ implies $f(x)\leqslant f(y)$; we speak of Schur-concavity when the inequality is reversed. Schur-convex or Schur-concave functions are necessary symmetric; symmetric function is Schur-convex if $\left(\frac{\partial f}{\partial x_i}-\frac{\partial f}{\partial x_j}\right)\left(x_i-x_j\right)\geqslant 0$ (Schur-Ostrowski criterion). For example, power sums $\sum_i x_i^q$ for $q\geqslant 1$ are Schur-convex. For more on \emph{majorization} and \emph{Schur-convexity} we refer to \cite{arnold2018majorization,shi2019schur}.

We define the \emph{moment domination} of a random variable $Y$ over $X$, denoted as $X\prec_{m} Y$, by requiring $\mathbf{E}X^q \leqslant \mathbf{E}Y^{q}$ for all positive integers $q$. In particular, it implies that MGF of $Y$ dominates the MGF of $X$, the majorization in the Lorentz stochastic order~\cite{arnold2018majorization}.



\section{Results}\label{sec:contrib}

\subsection{Main Result}
In the following result, we provide the promised numerically sharp, non-oblivious and geometrically insightful bounds for Rademacher random projections. 
In the (particularly interesting) case of sparse input, we obtain more explicit formulas involving binomial distributions.
\begin{theorem}[Sharp Moment Bounds for Rademacher Random Projections]\label{thm:main}
Let $\Phi$ be sampled according to the Rademacher scheme \eqref{eq:rademacher_matrix}, and define the distortion as
\begin{align}\label{eq:distortion}
    E(x) \triangleq \frac{ \|\Phi x \|^2 }{  \|x\|^2 } - 1.
\end{align}
Then the following holds true:
\begin{description}
\item[(a)] $E(x)$ has moments that are Schur-concave polynomials in $(x_i^2)$
\item[(b)] $E(x)$ is moment-dominated by $E_{*}$ defined as
\begin{align}
E_{*} & = \frac{1}{m}\sum_{i=1}^{m} (Z_i^2-1) 
\end{align}
where $Z_i$ are standardized binomial r.v.s.:
\begin{align}
   Z_i\sim^{IID} \frac{B-\mathbf{E}B}{\sqrt{\mathbf{Var}[B]}}, \ B\sim\mathsf{Binom}\left(\|x\|_0,\frac{1}{2}\right).
\end{align}
Equivalently,
\begin{align}
    \mathbf{E}E(x)^q \leqslant \mathbf{E}E_{*}^q
\end{align}
holds for $q=2,3,\ldots$ with equality when all components of the input $x$ are equal.
\end{description}
\end{theorem}

We briefly overview the proof of \Cref{thm:main} (see \Cref{fig:road_map}): it starts by a reduction to the dimension $m=1$, and writing the distortion as a Rademacher chaos of order 2; we then find extreme values of its moments geometrically, by means of \emph{Schur optimization}. Finally, these extreme values can be found explicitly and efficiently by linking them to binomial moments. 

\tikzstyle{my_style} = [draw, rectangle, rounded corners, minimum width={10cm}, 
fill={gray!10}]
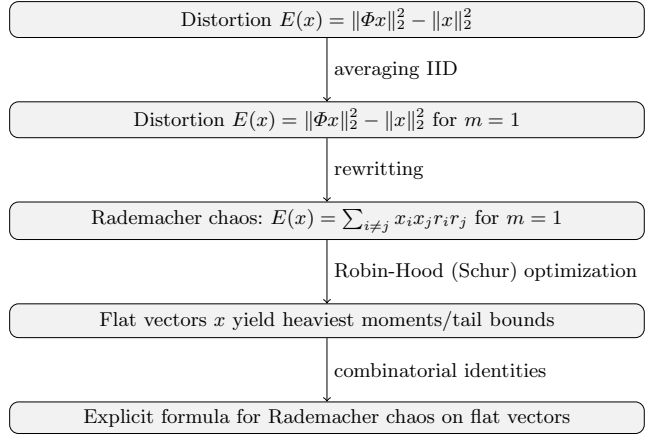
\begin{figure*}[!ht]
\centering
\resizebox{.7\textwidth}{!}{

\begin{tikzpicture}[node distance=2cm]
\node (high_dim) [my_style] {Distortion $E(x) = \|\Phi x\|_2^2-\|x\|_2^2$};
\node (1_dim) [my_style, below = 1cm of high_dim] {Distortion $E(x) = \|\Phi x\|_2^2-\|x\|_2^2$ for $m=1$};
\node (rademacher) [my_style, below = 1cm of 1_dim] {Rademacher chaos: $E(x) = \sum_{i\not=j}x_i x_j r_i r_j$ for $m=1$};
\node (rademacher_worst) [my_style, below = 1cm of rademacher] {Flat vectors $x$ yield heaviest moments/tail bounds};
\node (formula) [my_style, below = 1cm of rademacher_worst] {Explicit formula for Rademacher chaos on flat vectors};
\draw [<-] (1_dim) -- (high_dim) node[midway,right] {averaging IID};
\draw [<-] (rademacher) -- (1_dim) node[midway,right] {rewritting};
\draw [<-] (rademacher_worst) -- (rademacher) node[midway,right] {Robin-Hood (Schur) optimization};
\draw [<-] (formula) -- (rademacher_worst) node[midway,right] {combinatorial identities};
\end{tikzpicture}}
\caption{The proof roadmap for \Cref{thm:main}.}
\label{fig:road_map}

\end{figure*}

\subsection{Techniques: Proving Schur Convexity}

We present a useful framework for proving Schur convexity properties. It makes repeated use of few auxiliary facts to eventually reduce the task to a 2-dimensional problem. This is often easier than the classical approach of evaluating derivative tests.

\begin{lemma}\label{lemma:schur1}
The class of \emph{non-negative} Schur-convex (or concave) functions forms a semi-ring.
\end{lemma}

\begin{lemma}\label{lemma:schur2}
A multivariate function is Schur-convex (respectively, Schur-concave) if and only if it is symmetric and Schur-convex (respectively, Schur-concave) with respect to each pair of variables.
\end{lemma}

To demonstrate the usefulness of these facts, we sketch an alternative proof of a refined version of celebrated Khintchine's Inequality, due to Efron. This refinement plays an important role in statistics, namely in proving properties of popular Student-t tests.


\begin{corollary}[Refined  Khintchine Inequality~\cite{efron1968student}]
\label{cor:khintchine}
The mapping
\begin{align*}
x\rightarrow \mathbf{E}(\sum_i x_i r_i )^{q} 
\end{align*}
is a Schur-concave function of $(x_i^2)$. As a consequence for $\sigma=\|x\|_2$ we have
\begin{align*}
\mathbf{E}(\sum_{i=1}^{n} x_i r_i )^{q}  \leqslant \mathbf{E}\left(\frac{\sigma}{\sqrt{n}}\sum_{i=1}^{n} r_i \right)^{q} \leqslant \mathbf{E} \mathsf{Norm}(0,\sigma^2)^{q}.
\end{align*}
\end{corollary}
\begin{proof}
The symmetry with respect to $(x_i)$ is obvious. Applying 
the multinomial expansion to $(\sum_i x_i r_i )^{q}$, taking the expectation and using the symmetry of Rademacher random variables, we conclude that $\mathbf{E}(\sum_i x_i r_i )^{q} $ is polynomial in variables $(x_i^2)$. By \Cref{lemma:schur2}, it suffices to prove the Schur-concavity property for $x_1^2,x_2^2$. By the binomial formula and the independence of $r_1,r_2$ from $(r_i)_{i>2}$, we see that $\mathbf{E}(\sum_i x_i r_i )^{q}  = 
\sum_{k}\binom{q}{k}\mathbf{E}(\sum_{i>2}x_i r_i)^{q-k}\cdot \mathbf{E}(x_1r_1+x_2r_2)^k$ is a combination of expressions
$\mathbf{E}(x_1r_1+x_2r_2)^k$ with coefficients $c_k = \binom{q}{k}\mathbf{E}(\sum_{i>2}x_i r_i)^{q-k}$ that are independent of $x_1,x_2$ and non-negative due to the symmetry of $r_i$. By \Cref{lemma:schur1}, it suffices to prove that
$F_k\triangleq \mathbf{E}(x_1r_1+x_2r_2)^k$  is a Schur-concave function of $x_1^2,x_2^2$. Define $G_k \triangleq \mathbf{E}(x_1r_1+x_2r_2)^k x_1 x_2 r_1r_2$. By $(x_1r_1+x_2r_2)^k = (x_1r_1+x_2r_2)^{k-2}(x_1^2+x_2^2+2x_1x_2r_1r_2)$ we have 
$F_k = (x_1^2+x_2^2)F_{k-2} + 2 G_{k-2}$ and 
$G_k = (x_1^2+x_2^2)G_{k-2}+2x_1^2x_2^2 F_{k-2}$. Since $x_1^2+x_2^2$ and $x_1^2 x_2^2$ are both Schur-concave in $x_1^2,x_2^2$, the Schur-concavity property of 
$F_{k},G_{k}$ is proven when it is proven for $k:=k-2$. By mathematical induction, it suffices to realize that $F_{0}=1,F_1 =0,G_1=1,G_2=x_1^2x_2^2$ are Schur-concave in $x_1^2,x_2^2$.

Let $\mathbf{1}_{n}$ be the vector of $n$ ones. The first inequality follows then by $\frac{\sum_{i=1}^{n} x_i^2}{n}\cdot \mathbf{1}_n \prec (x_1^2,\ldots,x_n^2)$, and is clearly sharp. Since $\frac{1}{n+1}\mathbf{1}_{n+1}\prec \frac{1}{n}\mathbf{1}_n$ and the Schur-concavity implies that $\mathbf{E}(\sum_{i=1}^{n}r_i/\sqrt{n})^q$ increases with $n$, the second inequality follows by the CLT.
\end{proof}

\subsection{Techniques: Rademacher Chaoses}

Of independent interests are the techniques used in this work. The first result analyzes the quadratic Rademacher chaos  geometrically. It is similar in the spirit of the results of Efron~\cite{efron1968student} and Eaton~\cite{eaton1970note}, which however concern only a first-order Rademacher chaos.

\begin{theorem}[Schur-concavity of Rademacher chaoses]\label{thm:schur_concave}
Let $(r_i)$ be a sequence of independent Rademacher random variables. Then the Rademacher chaos moment
\begin{align}
    \mathbf{R}_q(x) \triangleq \mathbf{E}\left(\sum_{i\not=j} x_i x_j r_i r_j\right)^q
\end{align}
is a Schur-concave function of $(x_i^2)$ for every positive integer $q$.
\end{theorem}

The second result is a recipe for explicitly computing the extreme moment values:
\begin{theorem}[Extreme moments of Rademacher Chaos]\label{thm:extreme_rademacher}
For any $x$ and $K=\|x\|_0$ the following holds:
\begin{align}
    \mathbf{R}_q(x) \leqslant \mathbf{R}_q(x^{*}),\quad x^{*} = 
    \underbrace{\left(\frac{\|x\|_2}{\sqrt{K}},\ldots,\frac{\|x\|_2}{\sqrt{K}}\right)}_{K \text{ times} },
\end{align}
and furthermore the explicit value of this bound equals 
\begin{align}
 \mathbf{R}_q(x^{*}) = \|x\|_2^{2q}\cdot  \mathbf{E}_{\bar{B}} ({\bar{B}}^2-1)^q,
\end{align}
where $\bar{B} = \frac{B-K/2}{\sqrt{K/4}}$ standardizes  the symmetric binomial distribution with $K$ trials $B$.
\end{theorem}

\subsection{Numerical Comparison}



The presented result is \emph{numerically optimal} and \emph{captures input sparsity}. It should be compared against the bounds from \cite{achlioptas2001database} and the no-go result from \cite{burr2018optimal}, as illustrated in \Cref{tab:compare}. To see that  our bounds are better than those in \cite{achlioptas2001database}, it suffices to use the Gaussian majorization argument to obtain a weaker bound
$E(x) \prec_m \frac{\sum_{i=1}^{m}(N_i^2-1)}{m}$ where $N_i$ are independent standard normal random variables, and use known sub-gamma tail bounds for chi-square distributions (for example, those developed in the monograph on concentration inequalities~\cite{boucheron2003concentration}).

To validate our findings, we performed the following experiments on both synthetic and real-world datasets.

\begin{table*}[t!]
\resizebox{\textwidth}{!}{
\begin{tabular}{c|l}
\toprule
Author & Result \\
\midrule
\cite{burr2018optimal} & $\max_{x}\mathbf{P}[|E(x)|>\epsilon]\geqslant 2\exp\left(-\frac{m\epsilon^2(1+o(1))}{4})\right)$ when $m\gg \epsilon^{-2}, n\gg 1$ \\
\cite{achlioptas2001database}     &  $\mathbf{P}[|E(x)|>\epsilon]\leqslant 2\exp\left(-\frac{m\epsilon^2}{4}\left(1-\frac{2}{3}\epsilon\right)\right)$  \\
\textbf{this paper}    &  $E(x)\prec_m\frac{\sum_{i=1}^{m} Z_i^2-1}{m}$, $Z_i\sim^{IID} \frac{B-\mathbf{E}B}{\sqrt{\mathbf{Var}[B}},\ B\sim \mathsf{Binom}\left(\|x\|_0,\frac{1}{2}\right)$ \\
\bottomrule
\end{tabular}
}
\caption{Bounds from this work (\Cref{thm:main}) compared with the best prior bounds~\cite{achlioptas2001database} and the 
sharp no-go results~\cite{burr2018optimal}. Our bounds imply those from prior work by "normal majorization" arguments (see the supplementary material).}
\label{tab:compare}
\end{table*}

\subsubsection{Synthetic dataset}
\Cref{fig:inputsparsity_vs_moment} and 
\Cref{fig:inputsparsity_vs_tail} demonstrate numerical improvements. The input sparsity is the key: random projections are seen \emph{less distorted when input data is sparse}.

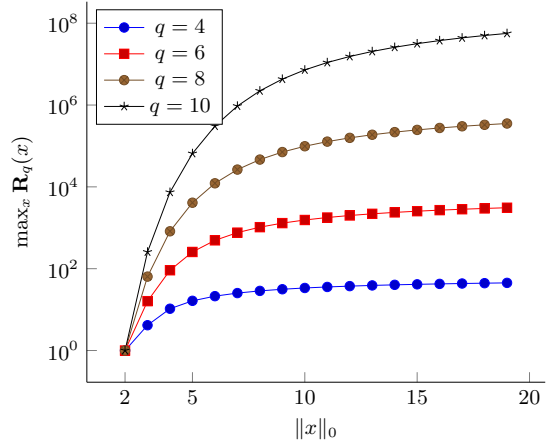
\begin{figure}[h]
\centering
\resizebox{.6\textwidth}{!}{
\begin{tikzpicture}
\begin{axis}[
    axis x line*=bottom,
    axis y line*=left,
    xtick={2,5,10,15,20},
    ymode=log,
    xlabel={$\|x\|_0$},
    ylabel={$\max_x \mathbf{R}_q(x)$},
    legend style={
        at={(0.02,0.98)},
        anchor=north west
    },
]
\addplot table [x=n, y=4, col sep=comma] {rademacher_extreme_moments.csv};
\addplot table [x=n, y=6, col sep=comma] {rademacher_extreme_moments.csv};
\addplot table [x=n, y=8, col sep=comma] {rademacher_extreme_moments.csv};
\addplot table [x=n, y=10, col sep=comma] {rademacher_extreme_moments.csv};
\legend{$q=4$,$q=6$,$q=8$,$q=10$};
\end{axis}
\end{tikzpicture}}
\caption{The more spread-out the input (controlled by sparsity $\|x\|_0$), the more distorted the projected output (captured by $\mathbf{R}_q(x)$, the Rademacher chaos moment). Utilizing the input dispersion improves probability bounds by orders of magnitude. Note: for normalization purposes, we assume $\|x\|_2=1$.}
\label{fig:inputsparsity_vs_moment}
\end{figure}

\begin{figure}[ht!]
\centering
\resizebox{.6\textwidth}{!}{
\begin{tikzpicture}
\begin{axis}[
    axis x line*=bottom,
    axis y line*=left,
    ymode=log,
    xlabel={$\epsilon$},
    ylabel={bound on $\mathbf{P}[|E(x)|>\epsilon]$},
    legend style={
        at={(0.98,0.98)},
        anchor=north east
    },
]
\addplot table [x=eps, y=new_10, col sep=comma] {bounds_compare.csv}; \addplot table [x=eps, y=new_20, col sep=comma] {bounds_compare.csv}; \addplot table [x=eps, y=new_30, col sep=comma] {bounds_compare.csv}; \addplot[mark=none,style=dashed] table [x=eps, y=old_10, col sep=comma] {bounds_compare.csv};
\legend{new($\ell=10$),new($\ell=20$),new($\ell=50$),old bounds};
\end{axis}
\end{tikzpicture}}
\caption{Capturing input-sparsity ($\ell=\|x\|_0$) improves the bounds on Rademacher random projections, as demonstrated by distortion  probability tails.
}
\label{fig:inputsparsity_vs_tail}
\end{figure}
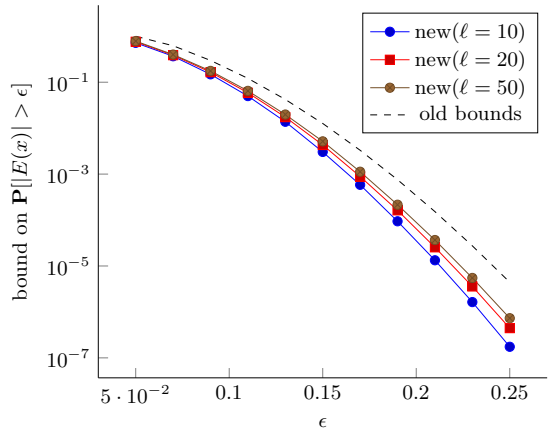

\begin{figure}[ht!]
\centering
\resizebox{.6\textwidth}{!}{
\begin{tikzpicture}
\begin{axis}[
    axis x line*=bottom,
    axis y line*=left,
    xtick={0.1,0.2,0.4,0.6,0.8,1.0},
    ymode=log,
    xlabel={embedding density $p$},
    ylabel={$distortion$},
    legend style={
        at={(0.02,0.98)},
        anchor=north west
    },
    legend pos=north east
]
\addplot table [x=n, y=mnist, col sep=comma] {real_world.csv};
\addplot table [x=n, y=optdigits, col sep=comma] {real_world.csv};
\addplot table [x=n, y=semeion, col sep=comma] {real_world.csv};
\addplot table [x=n, y=usps, col sep=comma] {real_world.csv};
\legend{$mnist$,$optdigits$,$semeion$,$usps$};
\end{axis}
\end{tikzpicture}}
\caption{The distortion measured w.r.t. the density of the embeddings shows that sparse data result in better bounds 
and proves that Rademacher projections are superior to sparse ones.
}
\label{fig:real-world}
\end{figure}
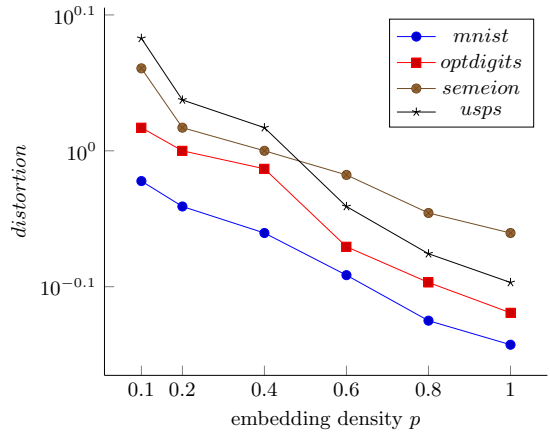

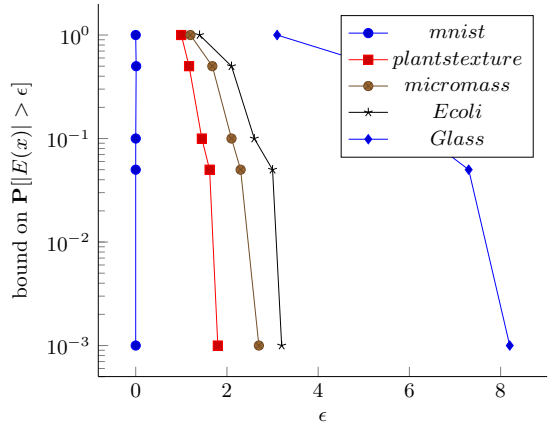
\begin{figure}[h!]
\centering
\resizebox{.6\textwidth}{!}{

\begin{tikzpicture}
\begin{axis}[
    axis x line*=bottom,
    axis y line*=left,
    xtick={0,2,4,6,8,10},
    ymode=log,
    xlabel={$\epsilon$},
    ylabel={bound on $\mathbf{P}[|E(x)|>\epsilon]$},
    legend style={
        at={(0.02,0.98)},
        anchor=north west
    },
    legend pos=north east
]
\addplot table [x=mnist, y = y, col sep=comma] {real_world_2.csv};
\addplot table [x=plantstexture, y=y, col sep=comma] {real_world_2.csv};
\addplot table [x=micromass, y=y, col sep=comma] {real_world_2.csv};
\addplot table [x=Ecoli, y=y, col sep=comma] {real_world_2.csv};
\addplot table [x=Glass, y=y, col sep=comma] {real_world_2.csv};
\legend{$mnist$, $plantstexture$, $micromass$, $Ecoli$, $Glass$};
\end{axis}
\end{tikzpicture}}
\caption{Measuring the distortion tail probability on real-world datasets confirms our theoretical findings: capturing the input-sparsity improves the bounds on Rademacher random projections. The datasets, from left to right, are displayed from the most sparse (mnist) to the least one (Glass). }
\label{fig:real-world_2}
\end{figure}

\subsubsection{Real-world datasets}

\Cref{fig:real-world} and \Cref{fig:real-world_2} show our findings. The results are validated with experiments performed on datasets from the SuiteSparse Matrix Collection (formerly the University of Florida Sparse Matrix Collection) available at  \url{https://sparse.tamu.edu/}. For these experiments, we selected matrices from machine learning datasets.

\section{Conclusions}\label{sec:conclusion}
We revisited the performance of Rademacher random projections, connecting the statistical guarantees with the input structure: for spreadness and, a special case, sparsity. The main result of this paper proves Schur-concavity properties, which makes the bounds numerically sharp and data aware (non-obliviuos) while giving a geometric perspective to the performance of the projections. 
We benchmarked our bounds both theoretically and empirically by measuring the distortion of the projected vectors against the original input data. As a result, dense projections are preferred, and they work incredibly well with sparse input data. We believe that our findings are of broader interest for a variety of statistical-inference applications.

\bibliographystyle{splncs04}
\bibliography{citations}

\appendix

\section{Proofs}\label{sec:proofs}

\subsection{Proof of Lemma 1}

Consider two non-negative functions $f,g$ and inputs $x\prec y$. Consider the identity
\begin{multline}
    f(y)g(y)-f(x)g(x) = (f(y)-f(x))\cdot g(y) + f(x)\cdot (g(y)-g(x)).
\end{multline}
If $f,g$ are Schur-convex then $f(y)-g(x)\geqslant 0$ and 
$g(y)-g(x)\geqslant 0$ and the whole expression is non-negative when $f,g$ are non-negative. This shows that $f\cdot g$ is also Schur-convex. The claim for Schur-concave functions follows analogously (the expression is then non-positive).

\subsection{Proof of Lemma 2}

The proof follows from the fact that $x$ is dominated by $y$ if and only if $x$ can be produced from $y$ by a sequence of \emph{Robin-Hood operations}, and the fact that Robin-Hood operations change only two fixed components of vectors. 

\subsection{Proof of Theorem 2}

\begin{proof}
Note that $\mathbf{R}_q$ is a polynomial in $x_i^2$ with integer coefficients, and thus a well-defined function of $(x_i^2)$. This follows by applying the multinomial expansion and noticing that monomials with odd exponents have expectation zero due to the symmetry of Rademacher distribution. $\mathbf{R}_q$ is obviously symmetric. By Lemma 2 it now suffices to validate the Schur-concativity for $x_1^2,x_2^2$ and any fixed choice of $(x_i)_{j>2}$.
Define the following expressions
\begin{align}
\begin{aligned}
P&=\sum_{i\not\in\{1,2\}}x_i r_i \\  R&= \sum_{i,j\not\in\{1,2\}}x_i x_j r_i r_j,
\end{aligned}
\end{align}
then our task is to prove the Schur-concativity of the function
\begin{align}
\mathbf{R}_q \triangleq \mathbf{E} \left( P(x_1 r_1 + x_2 r_2) + x_1 x_2 r_1 r_2 + R\right)^q,
\end{align}
with respect to $x_1^2,x_2^2$.

By the multinomial expansion we find that
\begin{align}
\begin{split}
\mathbf{R}_q \triangleq \sum_{q_1+q_2+q_3=q}\binom{q}{q_1,q_2,q_3}\bigg[ \mathbf{E}\left[ P^{q_1}R^{q_3}\right] \quad \mathbf{E}\left[ (x_1r_1+x_2r_2)^{q_1}(x_1x_2 r_1 r_2)^{q_2}\right] 
\bigg],
\end{split}
\end{align}
where we used the independence of $r_1,r_2$ on $(r_i)_{i>2}$ and thus also on $P,R$. Observe that 
$P^{q_1}R^{q_3}$ is, by definition and our assumption $x_i\geqslant 0$, a polynomial in symmetric random variables $r_i$ with non-negative coefficients. This observation shows that
\begin{align}
 \mathbf{E}\left[ P^{q_1}R^{q_3}\right] \geqslant 0,
\end{align}
and by Lemma 1 it suffices to prove that
\begin{align}
F\triangleq  \mathbf{E}\left[ (x_1r_1+x_2r_2)^{q_1}(x_1x_2 r_1 r_2)^{q_2}\right] 
\end{align}
is Schur-concave as a function of $x_1^2,x_2^2$ for any non-negative integers $q_1,q_2$.

To see that $F$ is indeed a well-defined function of $x_1^2,x_2^2$, note that it equals the expectation of a polynomial in the symmetric random variables $y_i= x_i r_i$; thus only monomials with even-degrees contribute, and the result is a polynomial in $y_i^2=x_i^2$. In fact, $F$ equals the sum of even-degree monomials in the expanded polynomial $(x_1+x_2)^{q_1}(x_1x_2)^{q_2}$.

We next observe that
\begin{align}\label{eq:reduction_2}
F = 
\begin{cases}
(x_1x_2)^{q_2}\mathbf{E}\left[ (x_1r_1+x_2r_2)^{q_1} \right] & q_2 \text{ even} \\
(x_1x_2)^{q_2-1}\mathbf{E}\left[ (x_1r_1+x_2r_2)^{q_1} x_1 x_2 r_1 r_2\right] & q_2 \text{  odd}
\end{cases}.
\end{align}
Note that $x_1 x_2$ is Schur-concave in non-negative $x_1,x_2$; indeed, the identity
$(x_1+\epsilon)(x_2-\epsilon) = x_1 x_2 + \epsilon(x_2-x_1-\epsilon)$ shows that Robin-Hood transfers increase the value. By Lemma 1 we conclude that $(x_1 x_2)^{k}$ is Schur concave in $x_1^2,x_2^2$ for non-negative even $k$. 
Thus, by  \Cref{eq:reduction_2} and Lemma 1 we conclude that it suffices to consider the case $q_2=1$, that is, to prove the Schur-concavity of the following two functions:
\begin{align}\label{eq:reduction_3}
G_k & \triangleq \mathbf{E}\left[ (x_1r_1+x_2r_2)^{k} \right] \\
H_k &\triangleq \mathbf{E}\left[ (x_1r_1+x_2r_2)^{k} x_1 x_2 r_1 r_2\right] .
\end{align}
with respect to $x_1^2,x_2^2$ for any non-negative integer $k$. 

Using the identity $(x_1r_1+x_2r_2)^{k} = (x_1r_1+x_2r_2)^{k-2}(x_1^2+x_2^2+2x_1x_2r_1r_2)$, we find the following recurrence relation
\begin{align}
G_k &= (x_1^2+x_2^2) G_{k-2}+2H_{k-2} \\
H_k &= 2x_1^2x_2^2 G_{k-2}+(x_1^2+x_2^2) H_{k-1} ,
\end{align}
valid for $k\geqslant 2$. Since
$x_1^2+x_2^2$ and $x_1^2 x_2^2$ are Schur-concave as functions of $x_1^2,x_2^2$, by Lemma 1 the concavity property proven for $k-2$ implies that it is valid also for $k$.
By induction, it suffices to verify the case $k=0$ and $k=1$. But we see that
\begin{align}
\begin{aligned}
G_0& = 1 \\
G_1& = 0 \\
H_0& = 1 \\
H_1& = 2x_1^2x_2^2 \\
\end{aligned}
\end{align}
are all Schur-concave as functions of $x_1^2,x_2^2$. This completes the proof.
\end{proof}

\subsection{Proof of Theorem 3}

Without loss of generality, we assume that $\|x\|_2=1$. From Theorem 2 and the fact that $(x^2_i)$ majorizes  $({x^{*}_i}^2)$ we obtain
\begin{align}\label{eq:rademacher_sup}
\begin{aligned}
\max_{\|x\|_0 \leqslant K}  \mathbf{E}\left(\sum_{i<j}x_i x_i r_i r_j \right)^q &= 
\mathbf{E}\left(\sum_{i<j}x^{*}_i x^{*}_i r_i r_j \right)^q &= \mathbf{E}\left(\frac{1}{K}\sum_{1\leqslant i<j \leqslant K} r_i r_j \right)^q,
\end{aligned}
\end{align}

Observe that $r_i = 1-2b_i$ where $(b_i)$ is a sequence of independent Bernoulli random variables with parameter $\frac{1}{2}$. Therefore,
\begin{multline}\label{eq:rademacher_linear}
    \mathbf{E}\left(\sum_{i=1}^{K}r_i\right)^q =\\
\begin{aligned}
    &=^{(a)} \sum_{k\in\mathbb{Z}} k^{q} \cdot \mathbf{P}\left\{\sum_{i=1}^{K}r_i = k\right\}   \\   
     &=^{(b)} \sum_{k} k^{q} \cdot \mathbf{P}\left\{\sum_{i=1}^{K}b_i = \frac{K-k}{2}\right\}  \\   
     &=^{(c)} \sum_{i=0}^{K} (K-2i)^q\cdot \mathbf{P}\left\{ \mathsf{Binom}\left(K,\frac{1}{2}\right) = i \right\} \\
     &=^{(d)} \frac{1}{2^{K}}\sum_{i=0}^{K} \binom{K}{i}(K-2i)^q \\
    & =^{(e)}     \frac{1}{2^K}\sum_{i}\binom{K}{i}(-K+2i)^q,
\end{aligned}
\end{multline}
where in (a) we use the fact that $\sum_i r_i$ takes integer values, (b) follows by the identity $r_i = 1-2b_i$, 
(c) follows by $\mathsf{Binom}(K,1/2) \sim \sum_{i=1}^{K} b_i$, (d) uses the explicit formula on the binomial probability mass function, and finally in (e) we substitute $i:= K-i$ and use the symmetry of binomial coefficients 
$\binom{K}{i}=\binom{K}{K-i}$.

Using the above formula, we further calculate 
\begin{multline}
\mathbf{E}\left(\sum_{1\leqslant i\not=j \leqslant K} r_i r_j \right)^q =\\
\begin{aligned}
& =^{(a)}  \mathbf{E}\left(\left(\sum_{i=1}^{K}r_i\right)^2-\sum_{i=1}^{K}r_i^2 \right)^q \\
  &=^{(b)} \sum_{j}\binom{q}{j}(-K)^{q-j}\mathbf{E}\left(\sum_{i=1}^{K}r_i\right)^{2j} \\
  & =^{(c)} \frac{1}{2^{K}}\sum_{i,j} \binom{q}{j} \binom{K}{i}(-K+2i)^{2j}  (-K)^{q-j}\\
  & =^{(d)} \frac{(-K)^q}{2^K}\sum_i \binom{K}{i} \left(1-\frac{(-K+2i)^2}{K}\right)^q,
\end{aligned}
\end{multline}
where (a) follows by the square sum completion, (b) follows by the binomial formula and $r_i^2=1$, (c) follows directly by  \Cref{eq:rademacher_linear}, and (d) is obtained by algebraic rearrangements.

Inserting \Cref{eq:rademacher_linear} into \Cref{eq:rademacher_sup}, we arrive at
\begin{multline}\label{eq:rademacher_sup_simpler}
 \max_{x:\|x\|_0\leqslant K} \mathbf{E}\left(\sum_{1\leqslant i\not=j \leqslant K} x_i x_j r_i r_j \right)^q  = \frac{1}{2^{K}} \sum_{i=0}^{K} \binom{K}{i} \left(\frac{(-K+2i)^2}{K}-1\right)^q.
\end{multline}

To simplify further, let $Z\sim \frac{\mathsf{Binom}\left(K,\frac{1}{2}\right)-\frac{K}{2}}{\sqrt{\frac{K}{4}}}$ be the standardization of the symmetric binomial distribution. 
Denoting $i\sim \mathsf{Binom}\left(K,\frac{1}{2}\right)$ we have $Z^2 \sim \frac{\left(i-\frac{K}{2}\right)^2}{{\frac{K}{4}}} = \frac{(-K+2i)^2}{K}$, and
we can rewrite \Cref{eq:rademacher_sup_simpler} as follows:
\begin{align}
    \max_{x:\|x\|_0\leqslant K}  \mathbf{E}\left(\sum_{1\leqslant i\not=j \leqslant K} x_i x_j r_i r_j \right)^q =  \mathbf{E}_Z \left(Z^2-1\right)^q,
\end{align}
which finishes the proof.

\subsection{Proof of Theorem 1}

We have to prove that for the distortion $E(\cdot)$ defined as in \begin{align}\label{eq:distortion}
    E(x) \triangleq \frac{ \|\Phi x \|^2 }{  \|x\|^2 } - 1.
\end{align}the following inequality holds true:
\begin{align}
E(x) \leqslant E(y),\quad (y_i^2) \prec (x_i^2).
\end{align}
The proof goes through several reduction steps until Schur-concavity of few simple functions.

We first observe that it suffices to prove that the moments of the expression
\begin{align}
x \rightarrow \|\Phi x\|^2-\|x\|^2,
\end{align}
are Schur-concavity with respect to $(x_i^2)$. Indeed, since $(y_i^2) \prec (x_i^2)$ implies $\|x\|^2 = \sum_i x_i^2 = \sum_i y_i^2 = \|y\|^2$ we have
  $\mathbf{E} E(x)^q \leqslant \mathbf{E} E(y)^q$ if and only if
$\mathbf{E}(\|\Phi x\|^2-\|x\|^2)^q \leqslant \mathbf{E}(\|\Phi y\|^2-\|y\|^2)^q$, by the definition of $E$.

We first prove that the distortion of $m$-dimensional projections is the average of $m$ IID distortions of $1$-D projections. Observe that
\begin{align}\label{eq:decomposition}
    \|\Phi x\|^2 - \|x\|^2 = \sum_{k=1}^{m}\left( (\Phi_k x)^2 -   \mathbf{E}(\Phi_k x)^2 \right),
\end{align}
where $\Phi_k$ is the $k$-th row of $\Phi$; this follows by $\mathbf{E}( \Phi_k x)^2 = \sum_{i} x_i^2\mathbf{Var}[\Phi_{k,i}] = \frac{1}{m}\|x\|^2$. Furthermore, the summands in \eqref{eq:decomposition} are independent and identically distributed:
\begin{align}\label{eq:average}
    (\Phi_k x)^2 - \mathbf{E}(\Phi_k x)^2 \sim \frac{1}{m}\sum_{i\not=j} x_i x_j r_i r_j.
\end{align}

Then we note that the Schur-concativity test can be done on the 1-D case. This follows because, due to the multinomial expansion applied to \Cref{eq:average}, the $q$-th moment of $m$-dimensional distortion is a multivariate polynomial in 1-D distortion moments of order $k\leqslant q$, with non-negative coefficients; the distortion moments are themselves non-negative, and by Lemma 1 and Theorem 2 we obtain the first part of the theorem.

Finally, applying Theorem 3 proves the second part.

\end{document}